\documentclass{article}

\usepackage{PRIMEarxiv}

\usepackage{algorithm}
\usepackage{algorithmic}

\usepackage[numbers]{natbib}
\usepackage{microtype}
\usepackage{graphicx}
\usepackage{booktabs} %
\usepackage{hyperref}
\hypersetup{
    colorlinks=true,
    linkcolor=blue,
    citecolor=blue,
    filecolor=magenta,
    urlcolor=cyan,
    pdftitle={Overleaf Example},
    pdfpagemode=FullScreen,
}

\usepackage{tikz}
\usetikzlibrary{bayesnet}
\usetikzlibrary{arrows}
\usetikzlibrary{positioning}
\usepackage{subcaption}
\usepackage{amsmath}
\usepackage{amssymb}
\usepackage{mathtools}
\usepackage{amsthm}
\usepackage{makecell}
\usepackage{amsfonts}
\usepackage{bm}
\usepackage{thmtools} 
\usepackage{thm-restate}
\usepackage{enumitem}
\usepackage{authblk}
\usepackage{longtable}
\usepackage{wrapfig}
\usepackage{multirow}
\usepackage{adjustbox}
\usepackage{colortbl}
\newcolumntype{H}{>{\setbox0=\hbox\bgroup}c<{\egroup}@{}}
\usepackage{makecell}
\usepackage{array}
\usepackage{listings}

\usepackage[capitalize,noabbrev]{cleveref}

\theoremstyle{plain}
\newtheorem{theorem}{Theorem}[section]
\newtheorem{proposition}[theorem]{Proposition}

\newtheorem{corollary}[theorem]{Corollary}

\theoremstyle{definition}

\theoremstyle{remark}

\usepackage[textsize=tiny]{todonotes}


\title{Evaluating Stochasticity in Deep Research Agents
}

\usepackage{authblk}

\author[1]{Haotian Zhai}
\author[1]{Elias Stengel-Eskin}
\author[1]{Pratik Patil}
\author[1]{Liu Leqi}

\affil[1]{%
University of Texas at Austin\vspace{-15pt}\thanks{
Corresponding author: {\texttt {haotian.zhai@utexas.edu}, \texttt {leqiliu@utexas.edu}}
}}

\DeclareMathOperator{\Var}{{\rm Var}}
\newcommand{\tr}{\mathop{\mathrm{Tr}}}
\newcommand{\TV}{\mathop{\mathrm{TV}}}

\renewcommand{\tilde}{\widetilde}

\begin{document}
\maketitle

\begin{abstract}

Deep Research Agents (DRAs) are promising agentic systems that gather and synthesize information to support research across domains such as financial decision-making, medical analysis, and scientific discovery. 
Despite recent improvements in research quality (e.g., outcome accuracy when ground truth is available), DRA system design often overlooks a critical barrier to real-world deployment: \textit{stochasticity}. 
Under identical queries, repeated executions of DRAs can exhibit substantial variability in terms of research outcome, findings, and citations. 
In this paper, we formalize the study of stochasticity in DRAs by modeling them as information acquisition Markov Decision Processes. 
We introduce an evaluation framework that quantifies variance in the system and identify three sources of it: information acquisition, information compression, and inference.
Through controlled experiments, we investigate how stochasticity from these modules across different decision steps influences the variance of DRA outputs.
Our results show that reducing stochasticity can improve research output quality, with inference and early-stage stochasticity contributing the most to DRA output variance. 
Based on these findings, we propose strategies for mitigating stochasticity while maintaining output quality via structured output and ensemble-based query generation. 
Our experiments on \textit{DeepSearchQA} show that our proposed mitigation methods reduce average stochasticity by 22\% while maintaining high research quality.

\end{abstract}


\section{Introduction}\label{sec:intro}

Deep Research Agents (DRAs) are a specialized class of LLM-driven agents designed for autonomous knowledge discovery~\cite{huang2025deep}. 
By interleaving iterative tool execution with adaptive reasoning and planning~\cite{yao2022react,team2025tongyi}, these agents systematically retrieve and synthesize external data to fulfill complex, information-heavy research objectives~\cite{du2025deepresearch,gupta2025deepsearchqa,wei2025browsecomp}.

However, the deployment of DRAs is hindered by a fundamental lack of reliability.
The execution of a DRA involves an iterative cycle of search, reasoning, and synthesis, where minor generation or search variations can cascade into vastly different research conclusions. 
This stochasticity, manifesting as significant variability in research output, findings and citations across repeated executions of the same query, remains largely underexplored. 

The implications of this stochasticity span from general user experience to high-stakes professional applications, which depend largely on whether the user possesses the domain expertise to audit the system's outputs.
For lay users, DRAs are often used to inform day-to-day decisions -- for example, medical information seeking. Lacking the expertise to verify the underlying evidence, these users often treat the DRA outputs as the ground truth. If a DRA provides conflicting advice on the safety of a medication across three different runs, the resulting stochasticity signals a lack of reliability, potentially leading to users following flawed advice.
For professional users, the challenge is operational. In fields like pharmaceutical R\&D or quantitative finance, DRAs are employed to automate high-volume workflows, such as synthesizing literature on protein folding or conducting systematic risk analyses. While these experts can differentiate between rigorous and unconvincing outputs, if a system produces a different risk profile for the same financial instrument each time it is queried, the expert must manually verify every execution. In this context, stochasticity transforms a tool meant for efficiency into a liability, as the time saved by automation is reclaimed by the necessity of rigorous validation.

To quantify and analyze the stochasticity in DRAs, we first formalize the execution of a deep research agent as an information acquisition Markov Decision Process (MDP), in which the agent iteratively performs query generation (information acquisition), summarization (information compression), and reasoning (inference) to uncover a set of task-relevant atomic findings and citations to source documents. 
This formulation provides a unifying abstraction for analyzing DRA's behavior across heterogeneous tasks and instantiations. 

Building on this formulation, we introduce a suite of metrics that characterize DRA stochasticity at multiple levels, including variability in final research outputs (e.g., answers), findings, and citation sources. 
These metrics enable fine-grained measurements that allow us to disentangle whether two runs disagree because they retrieve and cite different evidence, reason differently over the same evidence, or both.

We then adopt a measurement framework grounded in variance decomposition that attributes stochasticity in DRA’s evolving knowledge state to two sources: (i) \textit{propagated stochasticity} inherited from earlier states, and (ii) \textit{intrinsic stochasticity} arising from the current decision modules. 
The \textit{intrinsic stochasticity} term further decomposes into stochasticity from the information acquisition policy, the information compression policy, and the inference policy. 
This decomposition provides a conceptual lens for investigating where stochasticity enters the DRA system and how it compounds over time.
We empirically validate this framework using temperature ablations that selectively increase sampling temperature for individual modules at specific time steps while keeping all other modules deterministic.
This controlled analysis allows us to understand how stochasticity propagates in DRAs and why early-stage stochasticity dominates final output variance.

Finally, guided by findings from temperature ablations, we study practical mitigation strategies that reduce stochasticity while preserving research output quality. By combining structured summarization and reasoning output with ensembling early-stage search queries, we demonstrate that average stochasticity can be successfully reduced by 22\%,
while the research output quality is maintained.

The key contributions of this work are as follows: 

\begin{itemize}[leftmargin=7mm,itemsep=2pt]
    \item We identify the stochasticity problem and propose an information acquisition MDP formulation for deep research agents that supports principled analysis of stochasticity (Section~\ref{sec: modeling}).
    \item We introduce multi-level stochasticity metrics over answers, findings, and citations for evaluating stochasticity across executions (Section~\ref{sec:metrics}).
    \item We conceptually decompose DRA stochasticity into propagated and intrinsic stochasticity across agent modules (Section~\ref{sec:stochasticity_decomposition}), and introduce temperature ablation experiments to empirically investigate how these naturally entangled sources of uncertainty propagate and affect the final output stochasticity (Section~\ref{sec:temp_abl}).
    \item We propose mitigation strategies that substantially reduce stochasticity without degrading answer accuracy (Section~\ref{sec: miti}).
\end{itemize}

\section{Related Work}

\textbf{Deep research and browsing agents.}
The evolution of Deep Research Agents (DRAs) stems from early Retrieval-Augmented Generation (RAG) systems, which utilized large corpora to enhance factual accuracy in knowledge-intensive question answering \cite{lewis2020retrieval, gao2023retrieval}. 
While initial efforts focused on short-form factoid synthesis \cite{jin2025search, song2025r1, zheng2025deepresearcher,li2025search}, these methods often lack the reasoning depth required for comprehensive reports \cite{google2025gemini,openai2025deepresearch}. 
Frameworks such as GPTResearcher \cite{elovic2023gptr} have bridged this gap by integrating agentic workflows for long-form generation, utilizing query decomposition and hierarchical planning to ensure completeness. 
Contemporary systems such as OpenDeepSearch \cite{alzubi2025open} and Tongyi DeepResearch \cite{team2025tongyi} further extend these capabilities by interleaving iterative action-observation cycles with external tool use, such as Python execution. 
Other work such as FlashResearch \cite{nie2025flashresearch} also focuses on real-time agent orchestration, aiming at improve the efficiency of deep research agents. 
Collectively, these advancements transition AI from passive retrievers to autonomous agents capable of navigating complex, multi-stage research trajectories.

\textbf{Benchmarks and evaluation for DRAs.}
Early evaluations of research agents primarily relied on traditional Question Answering (QA) benchmarks \cite{wei2025browsecomp,gupta2025deepsearchqa,wu2025webwalker} to assess performance based on direct retrieval accuracy. 
While effective for validating answer precision, these benchmarks often fail to capture the complexity of synthesizing comprehensive research reports. 
To address the demands of open-ended, long-horizon research, recent works have introduced specialized benchmarks \cite{du2025deepresearch,xu2025researcherbench,yao2025rigorous,wang2025liveresearchbench} to evaluate deep research agents. 
These benchmarks move beyond simple accuracy, employing metrics based on content quality and factual reliability by decomposing final reports into atomic findings \cite{min2023factscore} and utilizing LLM-as-a-judge \cite{gu2024survey} to automate the assessment of the final report. 
Related work also develops specialized diagnostics for retrieval faithfulness and claim verification in DRA settings \cite{coelho2025deepresearchgym}. \citet{mustahsan2025stochasticity} studies stochasticity of evaluation for agentic systems, proposing statistical reliable measures.  

Although these methodologies have significantly advanced our ability to measure content quality and factual reliability, none of them quantify stochasticity in DRAs across runs. 
Our work bridges this gap by introducing a systematic framework to evaluate the stochasticity, attributing stochasticity to different components, and proposing stochasticity mitigation methods in DRAs.

\section{Modeling Deep Research Agents}
\label{sec: modeling}

We model a single execution run of a deep research agent as an \emph{information acquisition MDP}.
Let $\mathcal{Q}$ denote the set of all problem instances (user queries). 
For a fixed instance $q \in \mathcal{Q}$, let $\mathcal{F}(q)=\{f_1,\dots,f_N\}$ denote a finite universe of \emph{candidate atomic findings} relevant to $q$.

At step $t$, we track an abstract knowledge state $\mathbf{b}_t \in \{0,1\}^N$, where $\mathbf{b}_t[k]=1$ indicates that finding $f_k$ is supported by the evidence acquired so far.
We initialize $\mathbf{b}_0=\mathbf{0}$.
At each step, the agent either issues an information-seeking query or terminates.
Let the action space be $\mathcal{A}=\mathcal{S}\cup\{\mathtt{STOP}\}$, where $\mathcal{S}$ is the set of allowable search queries.
The agent samples an action
\begin{equation}
a_t \sim \pi_{\mathrm{query}}(\cdot \mid q, \mathbf{b}_t),
\end{equation}
where $\pi_{\mathrm{query}}$ is the \emph{query policy}.
If $a_t=\mathtt{STOP}$ the run terminates.
Otherwise, the environment (e.g., a search engine) returns retrieved content (documents, snippets, metadata, etc.), modeled as an observation
\begin{equation}
i_t \sim \beta_{\mathrm{env}}(\cdot \mid a_t),
\end{equation}
where $\beta_{\mathrm{env}}$ captures exogenous variability (e.g., dynamic content, API nondeterminism, etc.).

Retrieved content is typically too large and heterogeneous to store directly, so the agent compresses it into an intermediate representation (notes, extracted claims, structured tables, etc.).
We model this as
\begin{equation}
h_t \sim \pi_{\mathrm{sum}}(\cdot \mid q, \mathbf{b}_t, i_t),
\end{equation}
where $\pi_{\mathrm{sum}}$ is an \emph{information compression} policy.
Finally, the agent integrates the new compressed information to update its knowledge state via an \emph{inference} policy:
\begin{equation}
\mathbf{b}_{t+1} \sim \pi_{\mathrm{update}}(\cdot \mid q, \mathbf{b}_t, h_t, a_t).
\end{equation}
A trajectory of length $T$ is therefore
\[
\tau = (\mathbf{b}_0, a_0, i_0, h_0, \mathbf{b}_1, \dots, \mathbf{b}_T),
\]
terminating either by $\mathtt{STOP}$ or a maximum horizon.

This factorization isolates three key decision modules that are typically under the agent designer's control:
\begin{itemize}[leftmargin=7mm,nosep,itemsep=2pt]
    \item[(i)] \emph{Information acquisition} ($\pi_{\mathrm{query}}$): 
    whether or what to retrieve (query generating and stopping).
    \item[(ii)] \emph{Information compression} ($\pi_{\mathrm{sum}}$): 
    how retrieved content is distilled into a usable memory representation.
    \item[(iii)] \emph{Inference} ($\pi_{\mathrm{update}}$):
    how the agent integrates accumulated evidence to update supported findings.
\end{itemize}

The environment kernel $\beta_{\mathrm{env}}$ may also be stochastic (e.g., dynamic web search).
Our analysis framework assumes $\beta_{\mathrm{env}}$ is held fixed (e.g., via cached or reproducible search).

We can instantiate different DRA frameworks based on our modeling. In this work, we adopt ReAct~\cite{yao2022react, team2025tongyi} as a representative instantiation. 
In ReAct setup, the agent operates in a single, flat loop where the agent performs information acquisition, information compression and inference simultaneously. After ``reasoning'' about the current belief state (inference), the agent would ``act'' to generate search queries (information acquisition), and then uses another module to summarize the retrieved content (information compression) before ``reasoning'' of the next round begins. 
We also provide another instantiation example in Appendix \ref{sec:instantiation2}.

\begin{figure*}[!t]
    \centering
    \includegraphics[width=\linewidth]{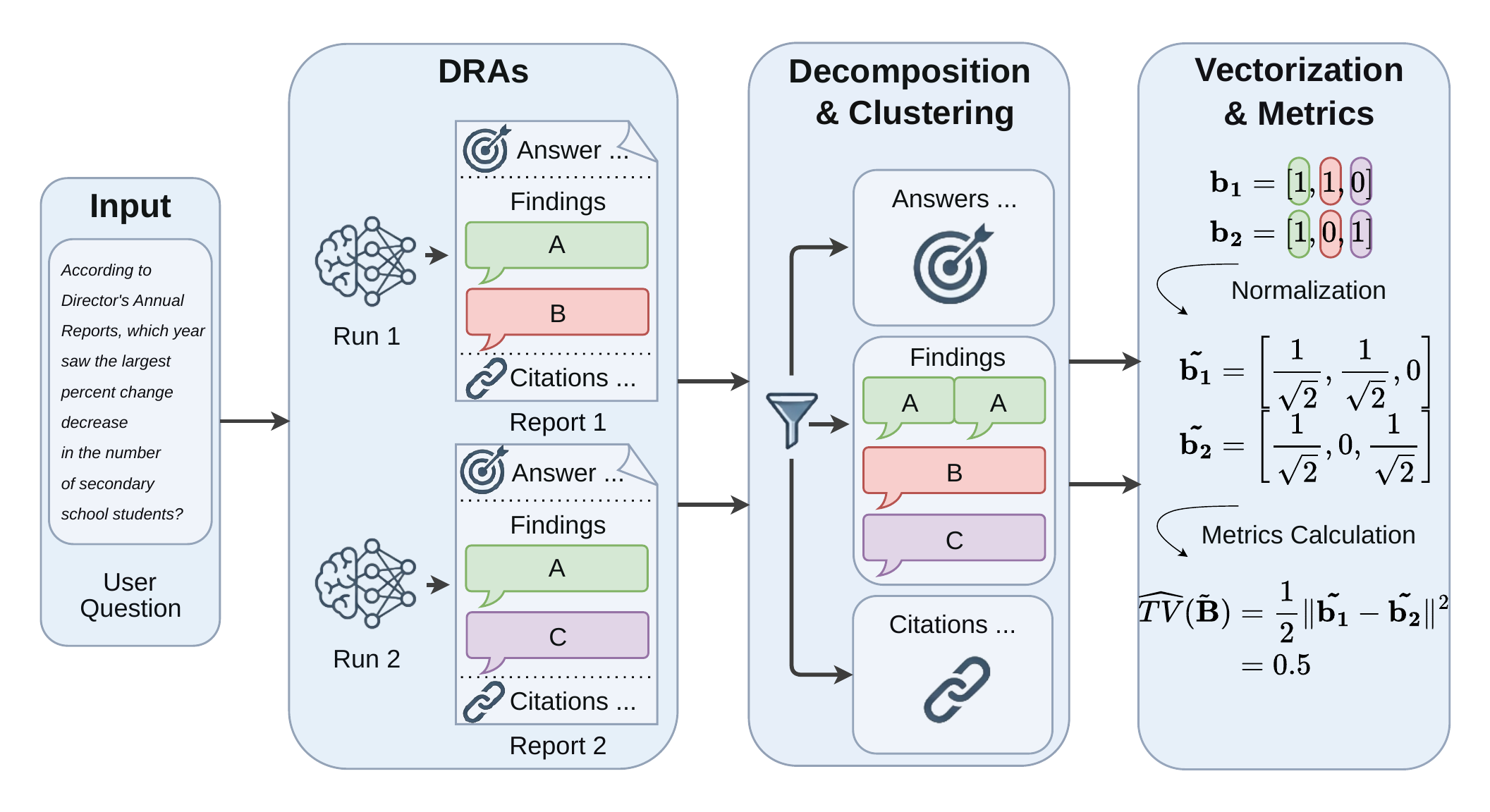}
    \caption{\textbf{Overview of the evaluation pipeline.} The process begins with a user question which triggers multiple independent Deep Research Agent (DRA) runs (in this example we use number of runs $k = 2$). The resulting reports are decomposed into answers, findings, and citations, and then clustered. After clustering, each report's answers, findings, and citations are mapped to binary vectors and normalized to compute the Total Variance (TV) as a measure of stochasticity. We expand findings as an example to illustrate the whole process. In reality, answers and citations are also processed in similar ways as in ~\cref{sec:metrics}.}
    \label{fig:pipeline}
\end{figure*}

\section{Evaluation Framework for DRA Stochasticity}
\label{sec:metrics}

We quantify stochasticity by running the same agent $n$ times on the same instance $q$ under identical configuation, producing trajectories $\mathcal{T}=\{\tau^{(1)},\dots,\tau^{(n)}\}$ and final reports $\{r^{(1)},\dots,r^{(n)}\}$.
For each report, we extract three objects of interest:
(i) an answer (when a verifiable answer exists),
(ii) a set of atomic findings, and
(iii) a set of citations (URLs).
We then cannonicalize these objects across runs and map each run to a vector representation that enables a unified variance-based metric. 

We have three random vectors of interest: the answer $\mathbf{Y}$, the citations $\mathbf{C}$, and the findings $\mathbf{B}$.
They are all categorical distributions 
and can be represented in similar ways. Note that for open-ended questions, we lack answers $\mathbf{Y}$.
Thus, in the following, we model the agent's output as a realization of a random vector $\mathbf{X} \in \{0, 1\}^d$. While the dimension $d$ and the semantic interpretation vary, the mathematical treatment of stochasticity remains unified. We describe the three specific instantiations of $\mathbf{X}$ as follows:

\begin{enumerate}[leftmargin=5mm,itemsep=2pt]
    \item \emph{Answers ($\mathbf{Y}$):} The answer per run is a categorical random variable modeled as a one-hot vector $\mathbf{y} \in \{0, 1\}^{K_A}$. Here, $K_A$ represents the total number of unique answers observed across all runs. The support of $\mathbf{Y}$ is the set of standard basis vectors $\{\mathbf{e}_1, \dots, \mathbf{e}_{K_A}\}$. Note that answers exist only for questions that are not open-ended. For open-ended questions, this can be replaced by some measurement on the quality of report.
    \item \emph{Findings ($\mathbf{B}$):} The findings per run are modeled as a binary vector $\mathbf{b} \in \{0, 1\}^{K_F}$. Here, $K_F$ is the size of the global union of unique findings across runs. The global set of findings is indexed from $1$ to $K_F$, such that a value of $1$ at the $k$-th entry of $\mathbf{b}$ indicates the presence of the $k$-th finding in that run.
    
    \item \emph{Citations ($\mathbf{C}$):} The citations per run are modeled as a binary vector $\mathbf{c} \in \{0, 1\}^{K_C}$. Here, $K_C$ is the size of the global union of unique citations across runs. The global set of citations is indexed from $1$ to $K_C$, such that a value of $1$ at the $k$-th entry of $\mathbf{c}$ indicates the presence of the $k$-th citation in that run.
\end{enumerate}

\subsection{Total Variance as a Measure of Stochasticity}

We seek a scalar that captures run-to-run variability.
We use the trace of the covariance (``total variance'') of the embedded output vector.
Let us recall a convenient identity.

\begin{proposition}
\label{theorem:the_1}
    Let $\mathbf{X}$ be a random vector in $\mathbb{R}^d$ with finite mean $\boldsymbol{\mu} = \mathbb{E}[\mathbf{X}]$ and covariance matrix $\boldsymbol{\Sigma} = \Var(\mathbf{X})$. 
    Let $\mathbf{X}_1$ and $\mathbf{X}_2$ be independent and identically distributed (i.i.d.) copies of $\mathbf{X}$.
    The covariance matrix $\boldsymbol{\Sigma}$ is given by:$$\boldsymbol{\Sigma} = \frac{1}{2} \mathbb{E} \left[ (\mathbf{X}_1 - \mathbf{X}_2)(\mathbf{X}_1 - \mathbf{X}_2)^\top \right].$$
\end{proposition}

\begin{corollary}[Total variance]
    We have
    \begin{equation}
    \label{eq:trvar_pairwise}
    \TV(\mathbf{X}) \coloneq \tr(\boldsymbol{\Sigma}) = \frac{1}{2} \mathbb{E} \left[ \|\mathbf{X}_1 - \mathbf{X}_2\|^2 \right].
    \end{equation}
\end{corollary}

The raw trace variance scales with vector magnitude and output size.
To make metrics comparable across runs and output types, we normalize each realization by its $\ell_2$ norm.
Let $\tilde{\mathbf{x}}_i = {\mathbf{x}_i}/{\|\mathbf{x}_i\|_2}$ be the normalized observation for run $i$ if $\|\mathbf{x}_i\|_2 > 0$ and $\mathbf{0}$ otherwise.
We then define (normalized) stochasticity as $\TV(\tilde{\mathbf{X}}) \coloneq \tr(\Var(\tilde{\mathbf{X}}))\in[0,1]$.

Given $n$ runs with normalized vectors $\{\tilde{\mathbf{x}}^{(i)}\}_{i=1}^n$, we estimate:

\begin{equation}
    \widehat{\TV}(\tilde{\mathbf{X}}) = \frac{1}{2 n(n-1)} \sum_{i = 1}^n \sum_{j=1}^n \|\tilde{\mathbf{x}}_i - \tilde{\mathbf{x}}_j\|^2.
    \label{eq:tv_estimator}
\end{equation}

This is a U-statistic estimator corresponding to \eqref{eq:trvar_pairwise} which is unbiased.
Before applying the estimator $\widehat{\TV}$ to our variables of interest, we must ensure that the vector dimensions are semantically consistent across the $n$ independent runs, i.e., dimension 3 for different $\tilde{\mathbf{x}}_i$ should represent the same answer/finding/citation.

For one-hot answers, $\widehat{\TV}(\tilde{\mathbf{X}})$ equals the empirical probability that two independent runs yield different canonical answers.
For binary findings/citations, normalization yields a cosine-overlap notion: the metric increases when runs share fewer canonical items relative to their sizes. We provide further interpretation and analysis of $\widehat{\TV}$ in Appendix \ref{app: eval}.

To capture the variability in \emph{how much} the agent outputs (e.g., number of findings or citations), we can also measure variance of the support size $\|\mathbf{X}\|_0$, using the same pairwise form as \eqref{eq:trvar_pairwise}:
\begin{equation}
    \hspace{-0.5em}
    \widehat{\TV}(\|\mathbf{X}\|_0) = \frac{1}{2 n(n-1)} \sum_{i = 1}^n \sum_{j=1}^n \left( \|\mathbf{x}_i\|_0 - \|\mathbf{x}_j\|_0 \right)^2.
\end{equation}

\subsection{Constructing Metrics from Agent Outputs}

To implement the proposed variance estimators on real agent outputs, we need to extract the required information from the agent's final report. 
We take the agent's final report, and adopt a decomposition procedure inspired by FactScore~\cite{min2023factscore}. We extract: (i) the final answer (for non–open-ended tasks), (ii) all URLs as citations, and (iii) a set of atomic findings obtained by decomposing the final report into minimal factual claims using LLM for each report. 
We include our prompts in Appendix \ref{app:prompts}. 

To ensure cross-run semantic alignment, findings from all trajectories are clustered into canonical findings using LLM-as-a-judge that determines whether two findings express the same underlying fact. Citations are canonicalized via URL normalization followed by exact string matching. After that, analogous binary vectors are constructed for answers and citations using their respective canonical sets. For non open-ended tasks, accuracy is computed by comparing the extracted final answer to the ground-truth answer using LLM-as-a-judge.  
We use greedy decoding for the reproducibility of evaluation pipeline. We demonstrate our evaluation pipeline in \Cref{fig:pipeline}.

\section{Analysis of Stochasticity via Variance Decomposition}

Following the metrics and evaluation pipeline we proposed in Section~\ref{sec:metrics},  
we model the progression of the random variable of interest over time, denoted as $\mathbf{X}_t$ (e.g., the accumulated findings or citations at step $t$). Since TV corresponds to the trace of the covariance matrix, it allows us to apply the Law of Total Variance to mathematically decompose the variance of $\mathbf{X}_{t+1}$, identifying how a policy ($\pi_{\text{query}}, \pi_{\text{sum}}, \pi_{\text{update}}$) introduce stochasticity at each step.

\subsection{Decomposing Stochasticity in Deep Research Agents}
\label{sec:stochasticity_decomposition}

As DRAs iteratively gather and process information, stochasticity accumulates at each decision step. While we model this process as an \textit{information acquisition MDP}, isolating the exact contribution of every variable during a full execution is computationally intractable. However, to understand and mitigate the variance in the final research output, we conceptually decompose the stochasticity at any given step $t$ into two primary components:

\textbf{Propagated Stochasticity.} 
This represents the stochasticity inherited from previous steps. Because a DRA's current action depends on its prior state (e.g., previous search results and intermediate reasoning), any stochasticity introduced early in the process naturally cascades. Conceptually, it captures how much of the stochasticity in the next state $\mathbf{X}_{t+1}$ is simply a result of the agent starting from a noisy or varying distribution at $\mathbf{X}_{t}$. Even if the agent were to act completely deterministically at step $t$, the stochasticity in the incoming state would still cause variability in the subsequent trajectory.

\textbf{Intrinsic Stochasticity.} 
This is the new, step-specific stochasticity introduced by the agent's internal policies during the current time step $t$. It measures the stochasticity that remains even if the previous state $\mathbf{X}_t$ were known with absolute certainty. Given the exact same starting state, the LLM-driven components of the DRA will still exhibit variability.

To provide a structural lens for auditing this uncertainty, we further categorize \textit{Intrinsic Stochasticity} based on the three distinct functional modules of the DRA at each step:

\begin{itemize}[leftmargin=5mm,nosep,itemsep=2pt]
    \item \emph{Information Acquisition ($\Delta_{\text{Query}}$):} The stochasticity introduced when the agent formulates search queries. Different phrasings of search queries can lead to vastly different retrieved documents and paragraphs.
    \item \emph{Information Compression ($\Delta_{\text{Sum}}$):} The stochasticity introduced when the agent parses, filters, and summarizes the raw retrieved content. Different extractions of facts from the exact same source document lead to divergent contexts.
    \item \emph{Inference ($\Delta_{\text{Update}}$):} The stochasticity introduced when the agent reasons over the newly synthesized information to update its internal belief state, answer open questions, or draw intermediate conclusions.
\end{itemize}

We provide the mathematical formulation of the stochasticity decomposition in Appendix \ref{sec:appendix_stochasticity_math}. In practice, perfectly disentangling these two sources of stochasticity is fundamentally intractable. Because of the autoregressive nature of the agent's trajectory, the intrinsic stochasticity introduced by any module at step $t$ immediately alters the resulting state $\mathbf{X}_{t+1}$. Consequently, this intrinsic stochasticity naturally becomes the propagated stochasticity for step $t+1$ and all subsequent steps. Furthermore, calculating the exact mathematical boundaries between them requires computing expectations over all possible future agent states and search engine responses, which is computationally infeasible for open-ended tasks.

Therefore, rather than attempting to compute exact analytical stochasticity terms, our empirical framework in the following section relies on targeted interventions. By temporally and modularly ablating the DRA's steps, we can approximate how these deeply entangled sources of uncertainty propagate and ultimately impact the stochasticity and quality of the final research outcome.

\begin{table*}[htbp]
    \centering
    \caption{\textbf{Full Ablation Results.} A comprehensive breakdown of stochasticity metrics across varying temperatures ($\lambda$), temporal injection points (Steps), and specific modules. This table corresponds to the experimental decomposition analyzed in Section~\ref{sec:stochasticity_decomposition}.}
    \label{tab:full_ablation}
    \tiny 
    \resizebox{0.95\textwidth}{!}{ 
    \begin{tabular}{ccl|cccccccc}
        \toprule
        \multirow{2}{*}{\textbf{$\lambda$}} & \multirow{2}{*}{\textbf{Step}} & \multirow{2}{*}{\textbf{Module}} & \multicolumn{3}{c}{\textbf{$\widehat{\text{TV}}(\tilde{\mathbf{X}})$}} & \multicolumn{2}{c}{\textbf{$\sqrt{\widehat{\text{TV}}(\|\mathbf{X}\|_0)}$}} & \multicolumn{2}{c}{\textbf{Mean Count ($\mu$)}} & \multirow{2}{*}{\textbf{Acc.}} \\
        \cmidrule(lr){4-6} \cmidrule(lr){7-8} \cmidrule(lr){9-10}
        & & & \textbf{Ans.} & \textbf{Find.} & \textbf{Cit.} & \textbf{Find.} & \textbf{Cit.} & \textbf{Find.} & \textbf{Cit.} & \\
        \midrule

        \multirow{12}{*}{0.5} 
        & \multirow{3}{*}{1} 
          & Query ($\pi_{\text{query}}$) & 0.37 & 0.76 & 0.42 & 29.97 & 2.10 & 90.82 & 6.62 & 0.40 \\
        & & Sum ($\pi_{\text{sum}}$)     & 0.58 & 0.84 & 0.52 & 32.01 & 2.17 & 90.84 & 6.46 & 0.40 \\
        & & Update ($\pi_{\text{update}}$)  & 0.59 & 0.85 & 0.55 & 31.45 & 2.49 & 88.80 & 6.96 & 0.52 \\ \cmidrule{2-11}
        
        & \multirow{3}{*}{2} 
          & Query & 0.36 & 0.70 & 0.30 & 27.35 & 2.30 & 92.46 & 5.96 & 0.46 \\
        & & Sum   & 0.36 & 0.65 & 0.38 & 30.85 & 2.60 & 88.14 & 6.48 & 0.44 \\
        & & Update& 0.38 & 0.82 & 0.40 & 28.68 & 2.48 & 96.58 & 6.40 & 0.40 \\ \cmidrule{2-11}

        & \multirow{3}{*}{3} 
          & Query & 0.36 & 0.61 & 0.30 & 32.08 & 2.26 & 86.92 & 6.08 & 0.34 \\
        & & Sum   & 0.28 & 0.52 & 0.28 & 31.88 & 2.14 & 96.82 & 6.84 & 0.34 \\
        & & Update& 0.32 & 0.68 & 0.34 & 32.62 & 2.67 & 89.96 & 6.32 & 0.44 \\ \cmidrule{2-11}

        & \multirow{3}{*}{Combined} 
          & Query & 0.44 & 0.76 & 0.50 & 35.84 & 2.41 & 102.90 & 6.52 & 0.40 \\
        & & Sum   & 0.59 & 0.84 & 0.55 & 34.46 & 2.76 & 89.52 & 6.96 & 0.44 \\
        & & Update& 0.59 & 0.89 & 0.55 & 39.17 & 2.55 & 96.88 & 7.62 & 0.41 \\ 
        \midrule

        \multirow{12}{*}{1.0} 
        & \multirow{3}{*}{1} 
          & Query ($\pi_{\text{query}}$) & 0.51 & 0.80 & 0.46 & 25.98 & 2.71 & 91.38 & 6.12 & 0.36 \\
        & & Sum ($\pi_{\text{sum}}$)     & 0.53 & 0.88 & 0.51 & 38.32 & 2.51 & 88.46 & 6.60 & 0.46 \\
        & & Update ($\pi_{\text{update}}$)  & 0.62 & 0.89 & 0.59 & 28.22 & 2.48 & 88.10 & 6.70 & 0.46 \\ \cmidrule{2-11}
        
        & \multirow{3}{*}{2} 
          & Query & 0.43 & 0.79 & 0.45 & 25.85 & 1.93 & 93.53 & 6.74 & 0.42 \\
        & & Sum   & 0.37 & 0.54 & 0.29 & 24.57 & 1.63 & 93.77 & 6.02 & 0.46 \\
        & & Update& 0.40 & 0.85 & 0.48 & 28.48 & 2.14 & 84.60 & 6.62 & 0.40 \\ \cmidrule{2-11}

        & \multirow{3}{*}{3} 
          & Query & 0.34 & 0.63 & 0.32 & 30.65 & 1.90 & 83.54 & 6.30 & 0.50 \\
        & & Sum   & 0.30 & 0.53 & 0.29 & 31.29 & 2.04 & 94.47 & 6.09 & 0.38 \\
        & & Update& 0.41 & 0.63 & 0.38 & 37.55 & 2.68 & 101.00 & 6.70 & 0.40 \\ \cmidrule{2-11}

        & \multirow{3}{*}{Combined} 
          & Query & 0.49 & 0.87 & 0.51 & 36.14 & 2.70 & 89.50 & 6.52 & 0.50 \\
        & & Sum   & 0.61 & 0.92 & 0.62 & 34.93 & 2.67 & 88.93 & 6.07 & 0.41 \\
        & & Update& 0.59 & 0.88 & 0.58 & 39.49 & 2.30 & 88.74 & 6.00 & 0.56 \\ 
        \bottomrule
    \end{tabular}
    }
\end{table*}

\subsection{Empirical Investigation via Temperature Ablation}
\label{sec:temp_abl}

To empirically quantify these theoretical components, we leverage \textit{temperature} to control the variance of different policies. By increasing the temperature $\lambda$ for specific policy modules ($\pi_{\text{query}}$, $\pi_{\text{sum}}$, $\pi_{\text{update}}$)  at specific steps while setting $\lambda=0$ (greedy decoding) for others, 
we empirically investigate how these modules affect DRA's stochasticity over time.
We applied these temperature controls at distinct temporal stages: early (Step 1), middle (Step 2), late (Step 3), and combined (Steps 1--3). This setup allows us to trace how stochasticity introduced by specific modules at specific times impacts the variance of the final research output $\mathbf{X}_T$.

For the DRA system, we use ReAct realization based on Tongyi DeepResearch~\cite{team2025tongyi}. In order to see the relationship between accuracy and stochasticity, we use 20 instances from QA dataset WebWalkerQA~\cite{wu2025webwalker}. We use Qwen3-30B-A3B-Instruct-2507~\cite{yang2025qwen3} as our backbone LLM, and set the number of runs $k = 10$ for each instance. We use You.com search API that is designed to return consistent results for identical queries. All experiments are conducted within a short time window (50 hours) to minimize potential variation due to search index updates or webpage changes, thereby eliminating environmental stochasticity from retrieval. We summarize our experiment results in \cref{tab:full_ablation}.


\begin{figure}[!ht]
    \centering

    \begin{subfigure}{\textwidth}
        \centering
        \includegraphics[width=0.9\textwidth]{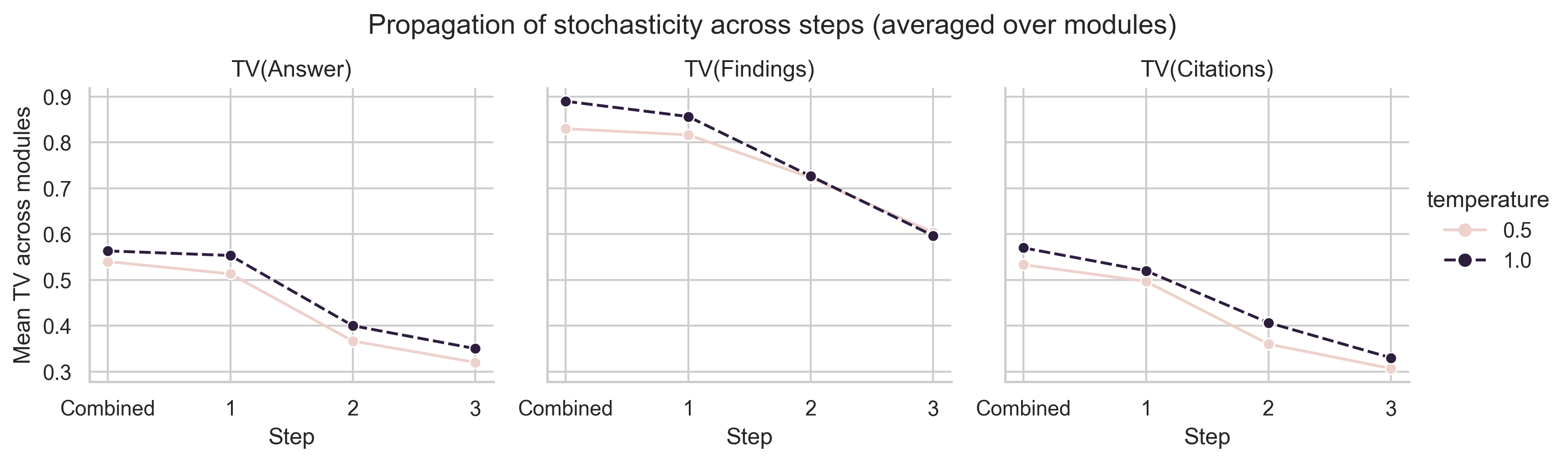}
        \caption{Propagation of stochasticity across steps.}
        \label{fig:step_propagation}
    \end{subfigure}

    \vspace{1em}

    \begin{minipage}{0.5\textwidth}
        \centering
        \begin{subfigure}{\linewidth}
            \centering
            \includegraphics[width=\linewidth]{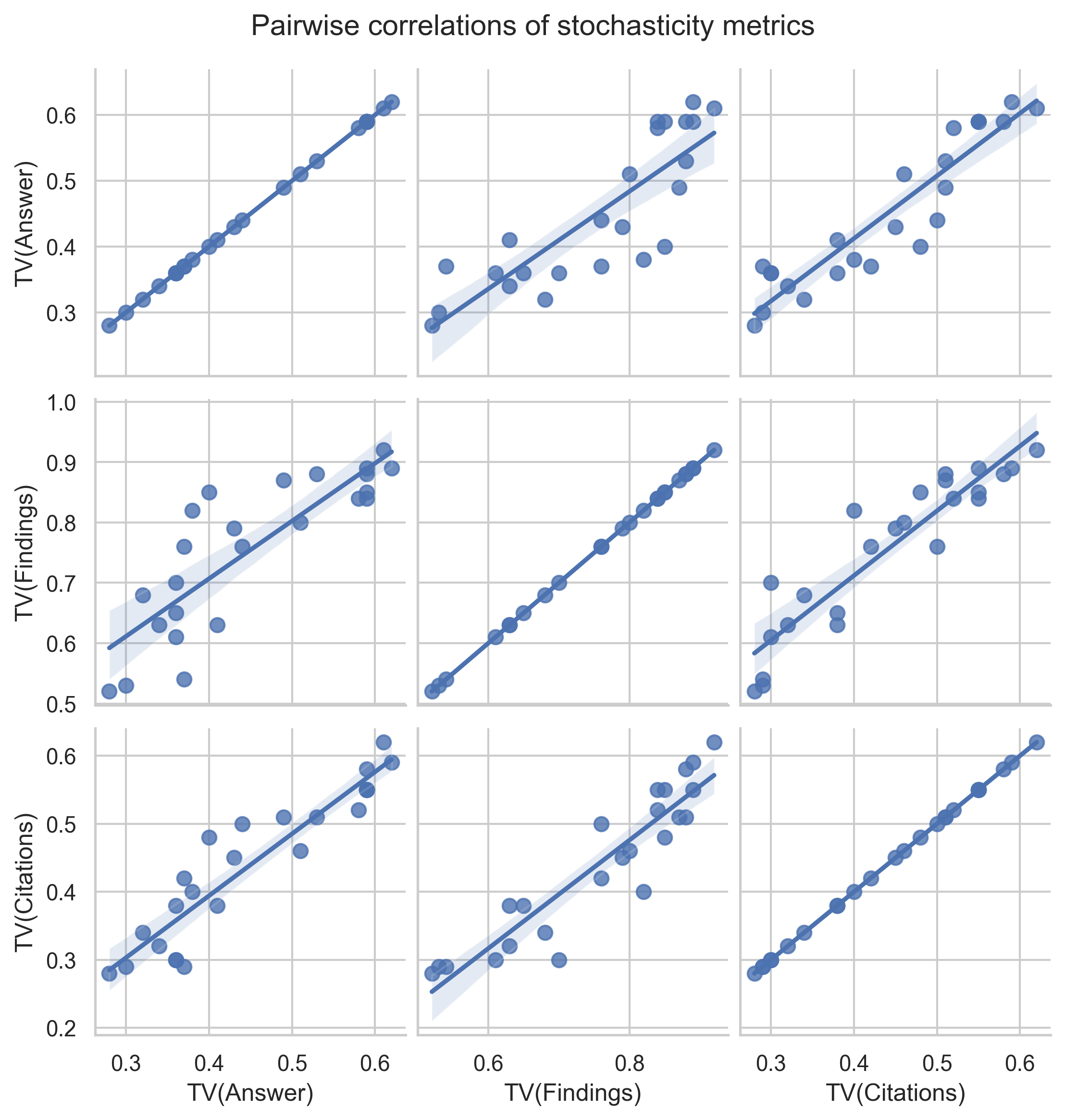}
            \caption{Pairwise correlations of stochasticity metrics.}
            \label{fig:tv_correlation}
        \end{subfigure}
    \end{minipage}
    \hfill
    \begin{minipage}{0.45\textwidth}
        \centering

        \begin{subfigure}{\linewidth}
            \centering
            \includegraphics[width=\linewidth]{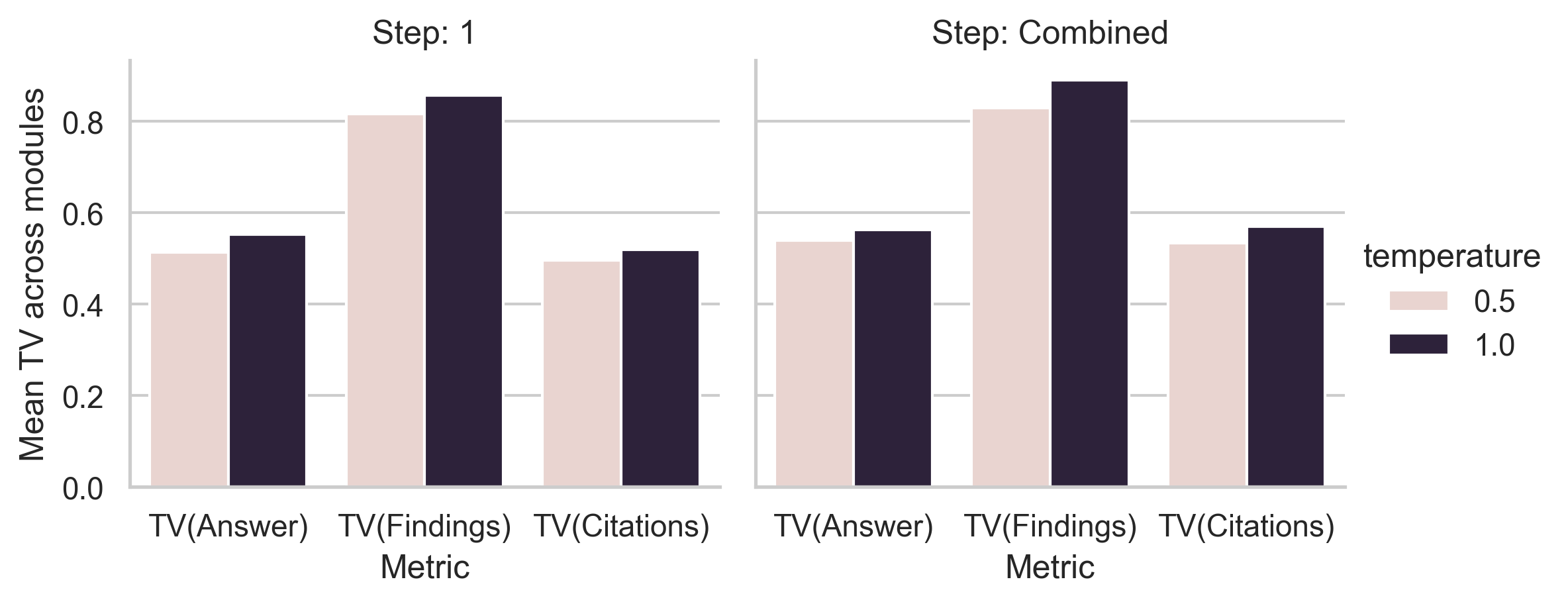}
            \caption{Effect of temperature on total variance.}
            \label{fig:temperature_effect}
        \end{subfigure}

        \vspace{1em}

        \begin{subfigure}{\linewidth}
            \centering
            \includegraphics[width=\linewidth]{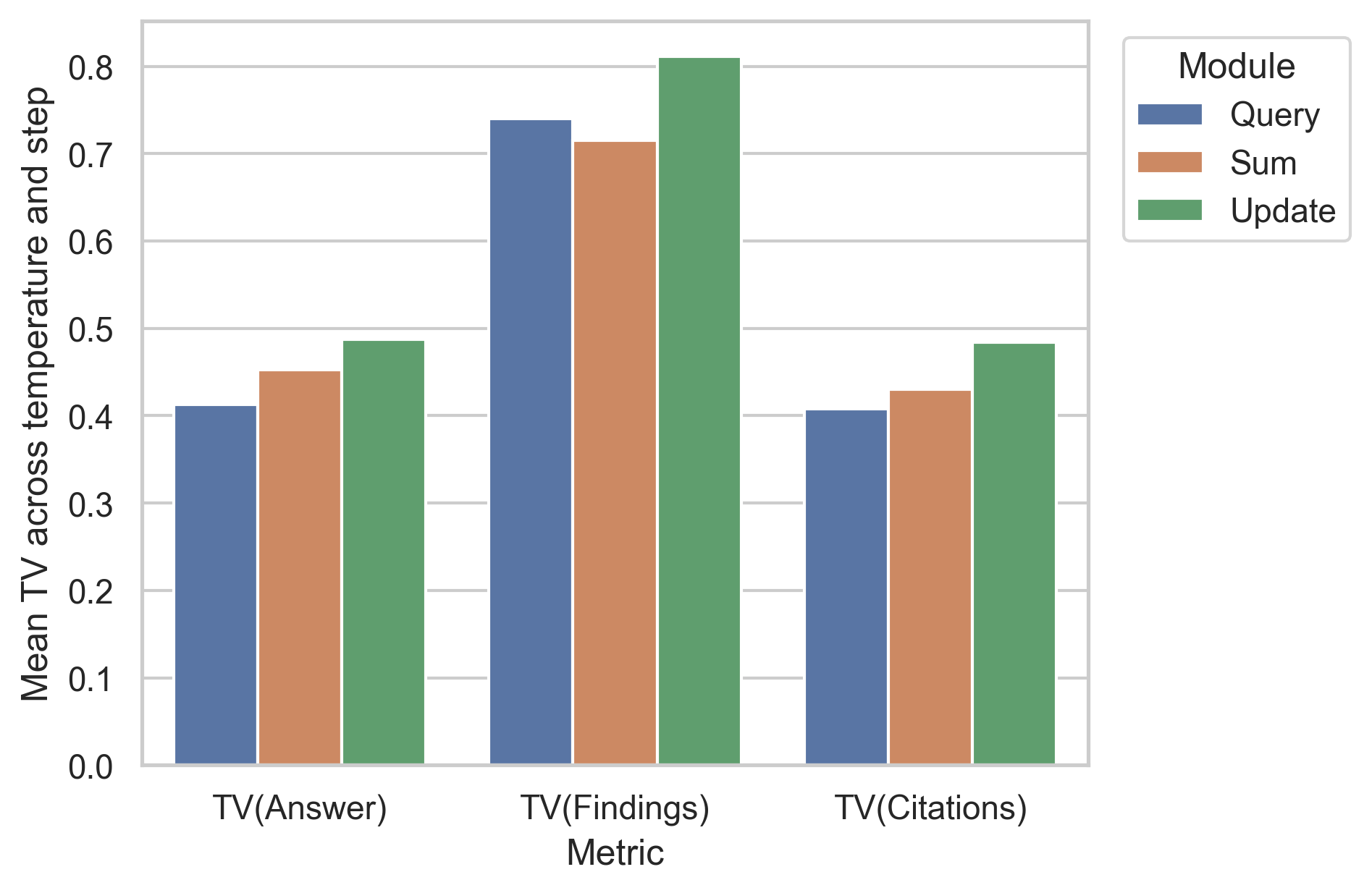}
            \caption{Module-wise contribution to stochasticity.}
            \label{fig:module_compare}
        \end{subfigure}

    \end{minipage}

    \caption{\textbf{Comprehensive Analysis of Stochasticity Behavior.}
    (a) Early-stage injections dominate propagation.
    (b) Strong positive correlations across answer, finding, and citation TV.
    (c) Higher sampling temperature increases total variance.
    (d) The Update module contributes the largest variance.}
    \label{fig:stochasticity_overview}
\end{figure}

\paragraph{Finding 1: Early-stage stochasticity influences final-stage stochasticity more than late-stage stochasticity.} 

To isolate temporal stochasticity effects independent of module choices, we average
$\widehat{\mathrm{TV}}(\tilde{\mathbf{X}})$ over the three modules
($\pi_{\text{query}}, \pi_{\text{sum}}, \pi_{\text{update}}$) for each perturbation step.
\cref{fig:step_propagation} plot the
resulting mean final variance for answer-level ($\mathbf{Y}$), finding-level
($\mathbf{B}$), and citation-level ($\mathbf{C}$) variances.

Across both temperatures and variance types, applying higher stochasticity to different modules at earlier steps consistently yields greater final variance than applying the same stochasticity at later stages. This empirical observation directly shows the significant influence of {propagated stochasticity} on the final variance. 
It confirms that uncertainty introduced in an initial state is magnified and ``carried forward'' to the total variance of the final output. 
The fact that the highest overall stochasticity occurs when policy temperatures are introduced at steps 1, 2, and 3 indicates that variance is cumulative, arising from the compounding of propagated and intrinsic stochasticity.
Consequently, stochasticity-control mechanisms are most effective when applied to early-stage modules to prevent this propagation.

\paragraph{Finding 2:  Findings, citations and answer stochasticity are positively correlated.}

To examine the relationship between different types of stochasticity, we plot
pairwise scatter plots between variance of findings $\widehat{\mathrm{TV}}(\mathbf{B})$, variance of citations
$\widehat{\mathrm{TV}}(\mathbf{C})$, and variance of answers $\widehat{\mathrm{TV}}(\mathbf{Y})$ across all temperatures, steps, and modules (\Cref{fig:tv_correlation}).
Each point corresponds to one experimental condition.

We observe strong positive correlations among all variance metrics $\widehat{\mathrm{TV}}(\mathbf{Y})$, $\widehat{\mathrm{TV}}(\mathbf{B})$, and $\widehat{\mathrm{TV}}(\mathbf{C})$,
confirming that these metrics capture a shared underlying notion of knowledge state stochasticity. 

\paragraph{Finding 3: Variance magnitude increases monotonically with temperature.}

To examine how stochasticity scales with sampling temperature, we compare
$\widehat{\mathrm{TV}}(\tilde{\mathbf{X}})$ under $\lambda=0.5$ and $\lambda=1.0$ while averaging
across the three modules. \cref{fig:temperature_effect} plot
the resulting mean final total variance for single-step
perturbations at Step~1 and cumulative perturbations across Steps~1--3, respectively.

Across both single-step and cumulative settings, higher sampling temperature consistently leads to larger estimated total variance. 
This monotonic increase indicates that temperature can act as a direct scaling factor for the total variance of the research output. 
Increasing $\lambda$ raises the entropy of the component policies ($\pi_{\text{query}}, \pi_{\text{sum}}, \pi_{\text{update}}$), thereby increasing intrinsic stochasticity at each step and the overall variance.

\paragraph{Finding 4: Higher stochasticity does not imply higher accuracy.}

Although increased stochasticity can encourage exploration, our results show that larger variance
$\widehat{\mathrm{TV}}(\tilde{\mathbf{X}})$ does not consistently lead to higher answer accuracy.
Table~\ref{tab:tv_vs_acc} shows examples where 
the final variance corresponds to similar or worse accuracy, as stochasticity increases.

\begin{table}[htbp]
\centering
\small
\caption{Examples showing that higher stochasticity does not monotonically improve accuracy. We select examples that have each one of $\lambda$, step, and module different.}
\label{tab:tv_vs_acc}
\begin{tabular}{c c c c c}
\toprule
$\lambda$ & Step & Module & $\widehat{\mathrm{TV}}(\mathbf{B})$ & Acc. \\
\midrule
0.5 & Step 1 & Query  & 0.76 & 0.40 \\
1.0 & Step 1 & Query  & 0.80 & 0.36 \\
\midrule
0.5 & Combined & Update & 0.89 & 0.41 \\
1.0 & Combined & Update & 0.88 & 0.56 \\
\midrule
0.5 & Step 1 & Sum & 0.84 & 0.40 \\
0.5 & Step 3 & Sum & 0.52 & 0.34 \\
\bottomrule
\end{tabular}
\end{table}

Across these examples, we observe that larger $\widehat{\mathrm{TV}}(\mathbf{B})$ do not
consistently imply higher accuracy, indicating that stochasticity alone is not a
reliable driver of performance.

\paragraph{Finding 5: Findings are more stochastic than citations.}

We compare stochasticity at the level of internal findings ($\mathbf{B}$) and external
citations ($\mathbf{C}$). Table~\ref{tab:find_vs_cit} reports averages
across $\lambda$, step, and module.

\begin{table}[htbp]
\centering
\caption{Average stochasticity of findings vs.\ citations 
across modules, steps, and temperatures.}
\label{tab:find_vs_cit}
\begin{tabular}{c|cc}
\toprule
Metric & Findings ($\mathbf{B}$) & Citations ($\mathbf{C}$) \\
\midrule
$\widehat{\mathrm{TV}}(\tilde{\mathbf{X}})$        & 0.76 & 0.44 \\
\bottomrule
\end{tabular}
\end{table}

Findings ($\mathbf{B}$) exhibit substantially larger total variance than citations ($\mathbf{C}$). 
This suggests that even though DRAs retrieve relatively consistent evidence sources (indicated by lower citation variance), the process of internalizing that evidence through information compression and inference policy introduces a significant amount of stochasticity. 

\paragraph{Finding 6: The inference module has a greater impact on the final stochasticity than the information acquisition and compression modules.}

To compare the effect of different modules causing intrinsic stochasticity, we average the stochasticity metrics across time for the three policy modules (\cref{fig:module_compare}). We find that adding stochasticity to $\pi_{\text{update}}$ yields the highest final output stochasticity across all three metrics. 
This suggests that mitigating stochasticity during inference (belief update) is more impactful for achieving overall system stability than controlling variances for the information acquisition and compression stages.

\section{Mitigation Attempt for Stochasticity}
\label{sec: miti}

Based on our findings, we provide initial strategies to 
lower stochasticity in DRA systems while maintaining (or increasing) their accuracy. In order to make our methods generalizable, we switch from a controlled environment to calling backbone LLMs through APIs. A significant challenge in this setting is the non-determinism in API LLM inference. Unlike local inference, API inference is non-batch invariant~\cite{he2025nondeterminism}, meaning that setting temperature to be $0$ is insufficient to eliminate stochasticity. 

Since we cannot reduce stochasticity under this more general setting through lowering temperature as in \ref{app:temp_tune}, we try to mitigate stochasticity while preserving accuracy through other algorithmic designs. We call Qwen3-235B-A22B-Instruct-2507-tput through Together AI API as backbone and use 20 instances from DeepSearchQA~\cite{gupta2025deepsearchqa} as a more complex dataset to test our mitigation strategies. We keep the number of runs $k = 10$ for each instance.
Throughout the experiments, we set the temperature to be $\lambda = 1$. We use the same search API settings as in \cref{sec:temp_abl}.

\subsection{Method 1: Structured Summarization and Reasoning Output}

To reduce the intrinsic stochasticity of the information compression ($\pi_{\text{sum}}$) and inference policy ($\pi_{\text{update}}$), we impose a structured output constraint for them. 
By forcing the model to generate outputs within a predefined JSON or Markdown schema, we expect to reduce the stylistic variations in the output, there by reducing $\Delta_{\text{Sum}}$ and $\Delta_{\text{Update}}$. 
An example of the structured output is shown in \ref{app:structured}.
As discussed in \textbf{Finding 6}, the stochasticity of the inference module has the greatest influence on the variance of the final output.
Therefore, we expect that imposing a structured output constraint on $\pi_{\text{update}}$ will reduce overall stochasticity more than imposing such a constraint on $\pi_{\text{sum}}$. 

\subsection{Method 2: Reducing Early-stage Query Stochasticity}

We also adopt a consensus-based ensemble for the information acquisition module ($\pi_{\text{query}}$). 
We issue $N$ independent sets of queries and retain only the intersection $a_t = \bigcap_{i=1}^N \text{queries}_i$. This ensures the agent only proceeds with queries that multiple runs have consensus in, effectively reducing early-stage query stochasticity $\Delta_{\text{Query}}$. To maintain efficiency, we decay $N \to 1$ as $t$ increases. 
Since early-stage stochasticity has a greater impact than later-stage stochasticity (\textbf{Finding 1}), we expect that this annealing over the number of runs will still effectively reduce variance. Finally, if no intersection exists, we use the first proposed set of queries.

\subsection{Experimental Result}

As shown in Table~\ref{tab:stochasticity_mitigation}, 
both mitigation methods reduce the stochasticity of the research output while maintaining (or sometimes increasing) the research quality. 
In addition, imposing structural output constraints on $\pi_{\text{update}}$ does reduce overall stochasticity more than imposing such a constraint on $\pi_{\text{sum}}$, reducing the average stochasticity by 5\%, and imposing structural output constraints for both modules yields more output stability than filtering search queries via intersection by an average of 6\%. 
By utilizing all mitigations together, we achieve the lowest stochasticity while having the highest accuracy. For Accuracy, we see a 12\% increase w.r.t Baseline, with an average 22\% decrease in stochasticity.
These results demonstrate that 
stochasticity can be effectively mitigated through algorithmic design
and controlling for stochasticity will not degrade the quality of the research output.

\begin{table}[!t]
    \centering
    \small
    \setlength{\tabcolsep}{3.5pt}
    \caption{Comparison of Stochasticity Mitigation Strategies. 
    The best performance in each category is highlighted in bold. 
    Struc.\ Comb.\ applies structured summarization and reasoning simultaneously; 
    Comb.\ integrates all three mitigation methods. Avg.\ TV represents the average value of the three-level stochasticity.}
    \label{tab:stochasticity_mitigation}
    \begin{tabular}{l c c c c c}
        \toprule
        \textbf{Method} 
        & \textbf{Acc.} 
        & \textbf{$\widehat{\mathrm{TV}}(\mathbf{Y})$} 
        & \textbf{$\widehat{\mathrm{TV}}(\mathbf{B})$} 
        & \textbf{$\widehat{\mathrm{TV}}(\mathbf{C})$}
        & \textbf{Avg. TV} \\
        \midrule
        Baseline       & 0.24 & 0.62 & 0.83 & 0.62 & 0.69 \\
        Struc. Sum.    & 0.28 & 0.58 & 0.80 & 0.64 & 0.67 \\
        Struc. Update  & 0.32 & 0.52 & 0.75 & 0.58 & 0.62 \\
        Struc. Comb.   & \textbf{0.36} & 0.44 & 0.68 & 0.56 & 0.56 \\
        Quer. Int.     & 0.32 & 0.50 & 0.74 & 0.61 & 0.62 \\
        \textbf{Comb.} & \textbf{0.36} & \textbf{0.38} & \textbf{0.61} & \textbf{0.43} & \textbf{0.47} \\
        \bottomrule
    \end{tabular}
\end{table}

\section{Conclusion}

In this paper, we provide a systematic analysis of stochasticity in DRAs, introducing principled metrics and a variance-decomposition framework for attributing stochasticity to specific modules and time steps. Our results show that stochasticity injected early in a trajectory propagates strongly and that inference is the dominant source of stochasticity. We also showcase the possibility of reducing stochasticity by proposing mitigation methods based on algorithmic designs. Overall, this work takes the first step toward building more reproducible deep research agents. Future directions include designing a broader range of mitigation strategies and developing deeper theoretical analyses of stochasticity in DRAs.

\bibliographystyle{plainnat}
\bibliography{main}  

@article{min2023factscore,
  title={Factscore: Fine-grained atomic evaluation of factual precision in long form text generation},
  author={Min, Sewon and Krishna, Kalpesh and Lyu, Xinxi and Lewis, Mike and Yih, Wen-tau and Koh, Pang Wei and Iyyer, Mohit and Zettlemoyer, Luke and Hajishirzi, Hannaneh},
  journal={arXiv preprint arXiv:2305.14251},
  year={2023}
}

@article{du2025deepresearch,
  title={DeepResearch Bench: A Comprehensive Benchmark for Deep Research Agents},
  author={Du, Mingxuan and Xu, Benfeng and Zhu, Chiwei and Wang, Xiaorui and Mao, Zhendong},
  journal={arXiv preprint arXiv:2506.11763},
  year={2025}
}

@article{wei2025browsecomp,
  title={Browsecomp: A simple yet challenging benchmark for browsing agents},
  author={Wei, Jason and Sun, Zhiqing and Papay, Spencer and McKinney, Scott and Han, Jeffrey and Fulford, Isa and Chung, Hyung Won and Passos, Alex Tachard and Fedus, William and Glaese, Amelia},
  journal={arXiv preprint arXiv:2504.12516},
  year={2025}
}

@inproceedings{yao2022react,
  title={React: Synergizing reasoning and acting in language models},
  author={Yao, Shunyu and Zhao, Jeffrey and Yu, Dian and Du, Nan and Shafran, Izhak and Narasimhan, Karthik R and Cao, Yuan},
  booktitle={The eleventh international conference on learning representations},
  year={2022}
}

@article{team2025tongyi,
  title={Tongyi deepresearch technical report},
  author={Tongyi DeepResearch, Team and Li, Baixuan and Zhang, Bo and Zhang, Dingchu and Huang, Fei and Li, Guangyu and Chen, Guoxin and Yin, Huifeng and Wu, Jialong and Zhou, Jingren and others},
  journal={arXiv preprint arXiv:2510.24701},
  year={2025}
}

@article{huang2025deep,
  title={Deep research agents: A systematic examination and roadmap},
  author={Huang, Yuxuan and Chen, Yihang and Zhang, Haozheng and Li, Kang and Zhou, Huichi and Fang, Meng and Yang, Linyi and Li, Xiaoguang and Shang, Lifeng and Xu, Songcen and others},
  journal={arXiv preprint arXiv:2506.18096},
  year={2025}
}

@misc{gupta2025deepsearchqa,
  title={DeepSearchQA: Bridging the Comprehensiveness Gap for Deep Research Agents},
  author={Gupta, Nikita and Chatterjee, Riju and Haas, Lukas and Tao, Connie and Wang, Andrew and Liu, Chang and Oiwa, Hidekazu and Gribovskaya, Elena and Ackermann, Jan and Blitzer, John and Goldshtein, Sasha and Das, Dipanjan},
  year={2025}
}

@article{alzubi2025open,
  title={Open deep search: Democratizing search with open-source reasoning agents},
  author={Alzubi, Salaheddin and Brooks, Creston and Chiniya, Purva and Contente, Edoardo and von Gerlach, Chiara and Irwin, Lucas and Jiang, Yihan and Kaz, Arda and Nguyen, Windsor and Oh, Sewoong and others},
  journal={arXiv preprint arXiv:2503.20201},
  year={2025}
}

@article{lewis2020retrieval,
  title={Retrieval-augmented generation for knowledge-intensive nlp tasks},
  author={Lewis, Patrick and Perez, Ethan and Piktus, Aleksandra and Petroni, Fabio and Karpukhin, Vladimir and Goyal, Naman and K{\"u}ttler, Heinrich and Lewis, Mike and Yih, Wen-tau and Rockt{\"a}schel, Tim and others},
  journal={Advances in neural information processing systems},
  volume={33},
  pages={9459--9474},
  year={2020}
}

@article{gao2023retrieval,
  title={Retrieval-augmented generation for large language models: A survey},
  author={Gao, Yunfan and Xiong, Yun and Gao, Xinyu and Jia, Kangxiang and Pan, Jinliu and Bi, Yuxi and Dai, Yixin and Sun, Jiawei and Wang, Haofen and Wang, Haofen},
  journal={arXiv preprint arXiv:2312.10997},
  volume={2},
  number={1},
  year={2023}
}

@article{jin2025search,
  title={Search-r1: Training llms to reason and leverage search engines with reinforcement learning},
  author={Jin, Bowen and Zeng, Hansi and Yue, Zhenrui and Yoon, Jinsung and Arik, Sercan and Wang, Dong and Zamani, Hamed and Han, Jiawei},
  journal={arXiv preprint arXiv:2503.09516},
  year={2025}
}

@article{song2025r1,
  title={R1-searcher: Incentivizing the search capability in llms via reinforcement learning},
  author={Song, Huatong and Jiang, Jinhao and Min, Yingqian and Chen, Jie and Chen, Zhipeng and Zhao, Wayne Xin and Fang, Lei and Wen, Ji-Rong},
  journal={arXiv preprint arXiv:2503.05592},
  year={2025}
}

@article{zheng2025deepresearcher,
  title={Deepresearcher: Scaling deep research via reinforcement learning in real-world environments},
  author={Zheng, Yuxiang and Fu, Dayuan and Hu, Xiangkun and Cai, Xiaojie and Ye, Lyumanshan and Lu, Pengrui and Liu, Pengfei},
  journal={arXiv preprint arXiv:2504.03160},
  year={2025}
}

@article{li2025search,
  title={Search-o1: Agentic search-enhanced large reasoning models},
  author={Li, Xiaoxi and Dong, Guanting and Jin, Jiajie and Zhang, Yuyao and Zhou, Yujia and Zhu, Yutao and Zhang, Peitian and Dou, Zhicheng},
  journal={arXiv preprint arXiv:2501.05366},
  year={2025}
}

@article{nie2025flashresearch,
  title={FlashResearch: Real-time Agent Orchestration for Efficient Deep Research},
  author={Nie, Lunyiu and Lipka, Nedim and Rossi, Ryan A and Chaudhuri, Swarat},
  journal={arXiv preprint arXiv:2510.05145},
  year={2025}
}

@misc{google2025gemini,
  author = {{Google Gemini Team}},
  title = {Gemini Deep Research: Your Personal Research Assistant},
  howpublished = {\url{https://gemini.google/overview/deep-research/}},
  year = {2025}
}

@misc{openai2025deepresearch,
  author = {{OpenAI}},
  title = {Introducing Deep Research},
  howpublished = {\url{https://openai.com/index/introducing-deep-research/}},
  year = {2025}
}

@misc{elovic2023gptr,
  author = {Assaf Elovic},
  title = {GPT Researcher: An autonomous agent that conducts deep research on any data},
  year = {2023},
  publisher = {GitHub},
  journal = {GitHub repository},
  howpublished = {\url{https://github.com/assafelovic/gpt-researcher}},
  commit = {master}
}

@article{wu2025webwalker,
  title={Webwalker: Benchmarking llms in web traversal},
  author={Wu, Jialong and Yin, Wenbiao and Jiang, Yong and Wang, Zhenglin and Xi, Zekun and Fang, Runnan and Zhang, Linhai and He, Yulan and Zhou, Deyu and Xie, Pengjun and others},
  journal={arXiv preprint arXiv:2501.07572},
  year={2025}
}

@article{xu2025researcherbench,
  title={Researcherbench: Evaluating deep ai research systems on the frontiers of scientific inquiry},
  author={Xu, Tianze and Lu, Pengrui and Ye, Lyumanshan and Hu, Xiangkun and Liu, Pengfei},
  journal={arXiv preprint arXiv:2507.16280},
  year={2025}
}

@article{yao2025rigorous,
  title={A Rigorous Benchmark with Multidimensional Evaluation for Deep Research Agents: From Answers to Reports},
  author={Yao, Yang and Wang, Yixu and Zhang, Yuxuan and Lu, Yi and Gu, Tianle and Li, Lingyu and Zhao, Dingyi and Wu, Keming and Wang, Haozhe and Nie, Ping and others},
  journal={arXiv preprint arXiv:2510.02190},
  year={2025}
}

@article{wang2025liveresearchbench,
  title={Liveresearchbench: A live benchmark for user-centric deep research in the wild},
  author={Wang, Jiayu and Ming, Yifei and Dulepet, Riya and Chen, Qinglin and Xu, Austin and Ke, Zixuan and Sala, Frederic and Albarghouthi, Aws and Xiong, Caiming and Joty, Shafiq},
  journal={arXiv preprint arXiv:2510.14240},
  year={2025}
}

@article{gu2024survey,
  title={A survey on llm-as-a-judge},
  author={Gu, Jiawei and Jiang, Xuhui and Shi, Zhichao and Tan, Hexiang and Zhai, Xuehao and Xu, Chengjin and Li, Wei and Shen, Yinghan and Ma, Shengjie and Liu, Honghao and others},
  journal={The Innovation},
  year={2024},
  publisher={Elsevier}
}

@article{coelho2025deepresearchgym,
  title={Deepresearchgym: A free, transparent, and reproducible evaluation sandbox for deep research},
  author={Coelho, Jo{\~a}o and Ning, Jingjie and He, Jingyuan and Mao, Kangrui and Paladugu, Abhijay and Setlur, Pranav and Jin, Jiahe and Callan, Jamie and Magalh{\~a}es, Jo{\~a}o and Martins, Bruno and others},
  journal={arXiv preprint arXiv:2505.19253},
  year={2025}
}

@article{mustahsan2025stochasticity,
  title={Stochasticity in Agentic Evaluations: Quantifying Inconsistency with Intraclass Correlation},
  author={Mustahsan, Zairah and Lim, Abel and Anand, Megna and Jain, Saahil and McCann, Bryan},
  journal={arXiv preprint arXiv:2512.06710},
  year={2025}
}

@article{yang2025qwen3,
  title={Qwen3 technical report},
  author={Yang, An and Li, Anfeng and Yang, Baosong and Zhang, Beichen and Hui, Binyuan and Zheng, Bo and Yu, Bowen and Gao, Chang and Huang, Chengen and Lv, Chenxu and others},
  journal={arXiv preprint arXiv:2505.09388},
  year={2025}
}

@article{he2025nondeterminism,
  author = {Horace He and Thinking Machines Lab},
  title = {Defeating Nondeterminism in LLM Inference},
  journal = {Thinking Machines Lab Blogs},
  year = {2025},
  note = {https://thinkingmachines.ai/blog/defeating-nondeterminism-in-llm-inference/}
}

\newpage
\appendix
\onecolumn

\section{Proofs} \label{app:proofs}

\subsection{Proof of \Cref{theorem:the_1}}

\begin{proof}
By definition, the covariance matrix is the expected outer product of the centered random variable:
$$\boldsymbol{\Sigma} = \mathbb{E} \left[ (\mathbf{X} - \boldsymbol{\mu})(\mathbf{X} - \boldsymbol{\mu})^\top \right] = \mathbb{E}[\mathbf{X}\mathbf{X}^\top] - \boldsymbol{\mu}\boldsymbol{\mu}^\top.$$
On the RHS, we have
\begin{align*} 
&\mathbb{E}\left[(\mathbf{X}_1 - \mathbf{X}_2)(\mathbf{X}_1^\top - \mathbf{X}_2^\top) \right] \\
&= \mathbb{E}[\mathbf{X}_1\mathbf{X}_1^\top] - \mathbb{E}[\mathbf{X}_1\mathbf{X}_2^\top] - \mathbb{E}[\mathbf{X}_2\mathbf{X}_1^\top] + \mathbb{E}[\mathbf{X}_2\mathbf{X}_2^\top].
\end{align*} 
Since $\mathbf{X}_1$ and $\mathbf{X}_2$ are i.i.d.,
the second moments are equal:
$$\mathbb{E}[\mathbf{X}_1\mathbf{X}_1^\top] = \mathbb{E}[\mathbf{X}_2\mathbf{X}_2^\top] = \mathbb{E}[\mathbf{X}\mathbf{X}^\top].$$
Since $\mathbf{X}_1$ and $\mathbf{X}_2$ are independent, 
we have the expectation of the product is the product of expectations.
$$\mathbb{E}[\mathbf{X}_1\mathbf{X}_2^\top] = \mathbb{E}[\mathbf{X}_1]\mathbb{E}[\mathbf{X}_2]^\top 
= \boldsymbol{\mu}\boldsymbol{\mu}^\top,$$
$$\mathbb{E}[\mathbf{X}_2\mathbf{X}_1^\top] = \mathbb{E}[\mathbf{X}_2]\mathbb{E}[\mathbf{X}_1]^\top 
= \boldsymbol{\mu}\boldsymbol{\mu}^\top.$$
Substituting these back into the expression for the RHS, we get
\begin{align*}
&\mathbb{E}\left[(\mathbf{X}_1 - \mathbf{X}_2)(\mathbf{X}_1^\top - \mathbf{X}_2^\top) \right] \\
&= \mathbb{E}[\mathbf{X}\mathbf{X}^\top] - \boldsymbol{\mu}\boldsymbol{\mu}^\top - \boldsymbol{\mu}\boldsymbol{\mu}^\top + \mathbb{E}[\mathbf{X}\mathbf{X}^\top] \\
&= 2\mathbb{E}[\mathbf{X}\mathbf{X}^\top] - 2\boldsymbol{\mu}\boldsymbol{\mu}^\top \\
&= 2 \left( \mathbb{E}[\mathbf{X}\mathbf{X}^\top] - \boldsymbol{\mu}\boldsymbol{\mu}^\top \right).
\end{align*}
This completes the derivation.
\end{proof}

\subsection{Formal Stochasticity Decomposition}
\label{sec:appendix_stochasticity_math}

\begin{proposition}[TV Decomposition]
\label{prop:total_decomposition}
Let $\mathbf{X}_t$ be the state at time $t$. The Total Variance at $t+1$, denoted $\text{TV}(\mathbf{X}_{t+1}) = \mathrm{Tr}(\mathrm{Var}(\mathbf{X}_{t+1}))$, decomposes into a propagated stochasticity term and an intrinsic stochasticity term $\mathcal{I}_t$:
\begin{equation}
    \text{TV}(\mathbf{X}_{t+1}) = \underbrace{\text{TV}_{\mathbf{X}_t}\left(\mathbb{E}[\mathbf{X}_{t+1} \mid \mathbf{X}_t]\right)}_{\text{Propagated Stochasticity}} + \underbrace{\mathbb{E}_{\mathbf{X}_t}[\mathcal{I}_t(\mathbf{X}_t)]}_{\text{Intrinsic Stochasticity}},
    \label{eq:total_decomposition}
\end{equation}
where the intrinsic stochasticity $\mathcal{I}_t(\mathbf{X}_t) = \mathrm{Tr}(\mathrm{Var}(\mathbf{X}_{t+1} \mid \mathbf{X}_t))$ can be further decomposed into three components controlled by the three policies, respectively:
\begin{equation}
    \mathcal{I}_t(\mathbf{X}_t) = \underbrace{\Delta_{\text{Query}}}_{\pi_{\text{query}}} + \underbrace{\Delta_{\text{Sum}}}_{\pi_{\text{sum}}} + \underbrace{\Delta_{\text{Update}}}_{\pi_{\text{update}}}.
    \label{eq:injection_decomposition}
\end{equation}
Assuming search results $i_t$ are deterministic given query $a_t$, these components are defined as the traces of their respective conditional variances:
\begin{align*}
    &\Delta_{\text{Query}} = \text{TV}_{a_t \sim \pi_{\text{query}}}\left(\mathbb{E}[\mathbf{X}_{t+1} \mid \mathbf{X}_t, a_t]\right), \\
    &\Delta_{\text{Sum}} = \mathbb{E}_{a_t}\left[\text{TV}_{h_t \sim \pi_{\text{sum}}}(\mathbb{E}[\mathbf{X}_{t+1} \mid \mathbf{X}_t, a_t, h_t])\right], \\
    &\Delta_{\text{Update}} = \mathbb{E}_{a_t, h_t}\left[\text{TV}_{\mathbf{X}_{t+1} \sim \pi_{\text{update}}}\left(\mathbf{X}_{t+1} \mid \mathbf{X}_t, a_t, h_t\right)\right].
\end{align*}
\end{proposition}

The decomposition presented in Proposition \ref{prop:total_decomposition} provides a structural lens through which we can audit the sources of uncertainty in the agent's state evolution. By applying the Law of Total Variance, we distinguish between uncertainty that is inherited from previous steps and uncertainty that is generated at the current step.

The first term, \textit{Propagated Stochasticity}, quantifies the `momentum' of prior stochasticity. Mathematically, it represents the variance of the conditional expectation; conceptually, it captures how much of the variance in $\mathbf{X}_{t+1}$ is simply a result of the agent starting from a noisy distribution $\mathbf{X}_{t}$. 

On the other hand, the \textit{Intrinsic Stochasticity} term, $\mathbb{E}_{\mathbf{X}_t}[\mathcal{I}_t(\mathbf{X}_t)]$, represents the stochasticity introduced during the current time step $t+1$. It measures the variance that remains even if the previous state $\mathbf{X}_t$ were known with absolute certainty. By labeling this as a \textit{intrinsic}, the framework treats the agent’s internal policies as sources of randomness that increases the stochasticity of the trajectory.

The decomposition of $\mathcal{I}_t(\mathbf{X}_t)$  isolates the contribution of each policy: 
\begin{itemize}[leftmargin=5mm,nosep,itemsep=2pt]
    \item $\Delta_{\text{Query}}$ measures the variance of the expected output caused by the agent's choice of information acquisition actions $a_t$.
    \item $\Delta_{\text{Sum}}$ reflects the variance introduced by the compression and synthesis of external search results into the summary $h_t$.
    \item $\Delta_{\text{Update}}$ captures the variance in the belief state update (inference), i.e., the stochasticity in integrating the synthesized information to infer the next state $\mathbf{X}_{t+1}$.
\end{itemize}

\subsubsection{Proof of Proposition~\ref{prop:total_decomposition}}

\begin{proof}
We apply the Law of Total Variance for random vectors, $\mathrm{Var}(\mathbf{X}) = \mathrm{Var}(\mathbb{E}[\mathbf{X}|\mathbf{Z}]) + \mathbb{E}[\mathrm{Var}(\mathbf{X}|\mathbf{Z})]$. Since the trace operator is linear, $\mathrm{Tr}(\mathrm{Var}(\mathbf{X})) = \mathrm{Tr}(\mathrm{Var}(\mathbb{E}[\mathbf{X}|\mathbf{Z}])) + \mathbb{E}[\mathrm{Tr}(\mathrm{Var}(\mathbf{X}|\mathbf{Z}))]$. We apply this recursively across the causal chain: $\mathbf{X}_t \to a_t \to h_t \to \mathbf{X}_{t+1}$.

\textbf{Isolation of History.} 
Conditioning on the previous state $\mathbf{X}_t$, we separate the variance propagated from the history distribution from the fresh variance introduced at step $t$:
\begin{align*}
    \text{TV}(\mathbf{X}_{t+1}) &= \mathrm{Tr}\left(\mathrm{Var}_{\mathbf{X}_t}\left(\mathbb{E}[\mathbf{X}_{t+1} \mid \mathbf{X}_t]\right)\right) \\
    &\quad + \mathbb{E}_{\mathbf{X}_t}\left[\text{TV}(\mathbf{X}_{t+1} \mid \mathbf{X}_t)\right]
\end{align*}
The second term is the expected intrinsic injection. We now decompose the inner term $\text{TV}(\mathbf{X}_{t+1} \mid \mathbf{X}_t)$.

\textbf{Isolation of Query Divergence ($\Delta_{\text{Query}}$).}
Conditioning on the action $a_t \sim \pi_{\text{query}}$, we apply the Law of Total Variance:
\begin{equation}
\begin{aligned}
\text{TV}(\mathbf{X}_{t+1} \mid \mathbf{X}_t)
&= \underbrace{\mathrm{Tr}\left(\mathrm{Var}_{a_t}(\mathbb{E}[\mathbf{X}_{t+1} \mid \mathbf{X}_t, a_t])\right)}_{\Delta_{\text{Query}}} \\
&\quad + \mathbb{E}_{a_t}\left[\text{TV}(\mathbf{X}_{t+1} \mid \mathbf{X}_t, a_t)\right]
\end{aligned}
\end{equation}
Here, $\Delta_{\text{Query}}$ captures the spread in the future state caused solely by the stochastic choice of queries.

\textbf{Isolation of Summarization Noise ($\Delta_{\text{Sum}}$).}
We decompose the residual term $\mathbb{E}_{a_t}[\text{TV}(\mathbf{X}_{t+1} \mid \mathbf{X}_t, a_t)]$. Since retrieved information $i_t$ is deterministic given $a_t$, the next stochastic component is the distilled findings $h_t \sim \pi_{\text{sum}}$. Conditioning on $h_t$:
\begin{equation}
\begin{aligned}
\text{TV}(\mathbf{X}_{t+1} \mid \mathbf{X}_t, a_t)
&= \underbrace{
\mathrm{Tr}\left(\mathrm{Var}_{h_t}(\mathbb{E}[\mathbf{X}_{t+1} \mid \mathbf{X}_t, a_t, h_t])\right)
}_{\text{Variance due to summarization}} \\
&\quad + \mathbb{E}_{h_t}\left[
\text{TV}(\mathbf{X}_{t+1} \mid \mathbf{X}_t, a_t, h_t)
\right]
\end{aligned}
\end{equation}
Taking the expectation over $a_t$ yields $\Delta_{\text{Sum}}$.

\textbf{Isolation of Update Instability ($\Delta_{\text{Update}}$).}
The remaining term represents the variance of the update step itself, given fixed history, action, and summary. This captures the inherent stochasticity of the reasoning policy $\pi_{\text{update}}$:
\[
\Delta_{\text{Update}} = \mathbb{E}_{a_t} \mathbb{E}_{h_t} \left[ \text{TV}(\mathbf{X}_{t+1} \mid \mathbf{X}_t, a_t, h_t) \right]
\]
Substituting these terms back into Eq.~\eqref{eq:total_decomposition} yields the full decomposition.
\end{proof}

\section{Prompts} \label{app:prompts}

\subsection{Claim Decomposition}

\lstset{
    showspaces=false,
    showstringspaces=false,
    basicstyle=\ttfamily\small,
    keywordstyle=\color{black},
    commentstyle=\color{black},
    stringstyle=\color{black},
    breaklines=true,
    frame=single,
    keepspaces=true
}

\begin{lstlisting}[language=Python, caption=Claim Decomposition]

## Task Description
Extract all factual claims from the provided report. Each claim should be a factual statement that
can be verified. Claims may or may not have supporting citations.

## Input
A Research Question and a complete report containing factual claims, some of which may have
citation markers and corresponding URLs (either inline or in a reference section).

## Output Requirements
- Extract each distinct factual claim throughout the entire report
- For each claim, output a JSON object with:
  - The exact claim text as a string
  - The original text from the report containing this claim (context)
  - The corresponding citation URL as source (if a citation marker directly follows the claim)
- If a claim has a citation marker directly following it, return the supporting URL as source
- If a claim does not have a citation marker directly following it, return an empty string for source
- Ensure all string values are properly escaped for valid JSON format (e.g., Replace internal quotation
  marks (") with escaped quotation marks (\\")) in the claim and context
- Return a JSON array containing all claim objects

## Format Specification
[
  {
    "claim": "The exact statement representing a factual claim",
    "context": "The original sentence or passage from the report containing this claim",
    "source": "https://example.com/source1"
  },
  {
    "claim": "Another factual statement without direct citation",
    "context": "The original sentence or passage from the report containing this claim",
    "source": ""
  }
]

## Guidelines for Claim Identification
1. A claim should be a complete, standalone factual statement
2. Maintain the original wording where possible, but remove unnecessary context
3. Extract all factual claims regardless of whether they have citation support
4. Only map citation markers (numbers, author names, etc.) to their corresponding URLs in
   the references section when the marker directly follows the claim statement
5. Exclude opinions, speculations, or methodological descriptions
6. Extract the context passage containing each claim for verification purposes
7. If multiple claims are associated with the same citation, extract them as separate entries

## Citation URL Mapping
- If URLs appear directly after claims, use those URLs directly
- Citation markers (e.g., a number or [number]) must directly follow the claim to be considered
  as supporting that claim
- If claims use citation markers that reference a bibliography or reference section, locate the corresponding
  URLs in that section
- If a claim has no directly following citation marker, use an empty string for source

\end{lstlisting}

\subsection{Atomic Finding Decomposition}

\begin{lstlisting}[caption=Atomic Fact Decomposition]
You are given a factual statement (a claim) from a technical report.
Break the claim down into independent, minimal atomic facts that can each be
verified in isolation. Keep each atomic fact short, declarative, and free of
conjunctions when possible. Avoid duplicating the same content in multiple ways
unless it clarifies distinct atomic facts (e.g., subject + membership vs subject + role).

Format:
- Return a JSON array of strings. Each string is one atomic fact.
- Do not include any extra commentary or Markdown. Only return the JSON array.

Examples:
Input: "He was an American composer, conductor, and musical director."
Output: [
  "He was an American.",
  "He was a composer.",
  "He was a conductor.",
  "He was a musical director."
]

Input: "She currently stars in the romantic comedy series, Love and Destiny, which premiered in 2019."
Output: [
  "She currently stars in Love and Destiny.",
  "Love and Destiny is a romantic comedy series.",
  "Love and Destiny premiered in 2019."
]

Input: "During his professional career, McCoy played for the Broncos, the San Diego Chargers, the Minnesota Vikings, and the Jacksonville Jaguars."
Output: [
  "McCoy played for the Broncos.",
  "McCoy played for the Broncos during his professional career.",
  "McCoy played for the San Diego Chargers.",
  "McCoy played for the San Diego Chargers during his professional career.",
  "McCoy played for the Minnesota Vikings.",
  "McCoy played for the Minnesota Vikings during his professional career.",
  "McCoy played for the Jacksonville Jaguars.",
  "McCoy played for the Jacksonville Jaguars during his professional career."
]

Input: "The EU approved the AI Act in 2024 and introduced new compliance requirements."
Output: [
  "The EU approved the AI Act in 2024.",
  "The AI Act introduced new compliance requirements."
]

Input: "The Amazon River is the largest by discharge and flows into the Atlantic Ocean."
Output: [
  "The Amazon River is the largest river by discharge.",
  "The Amazon River flows into the Atlantic Ocean."
]

Now decompose the following claim into atomic facts and return only a JSON array of strings:
\end{lstlisting}

\subsection{Answer Extraction}

\begin{lstlisting}[caption=Atomic Fact Decomposition]
## Task Description
Extract the answer to the research question from the provided report.

## Input
A Research Question and a complete report containing the answer.

## Output Requirements
Extract the direct answer to the research question
Provide supporting evidence/context from the report
Return a JSON object with the answer and supporting context

## Format Specification
{
  "question": "The research question",
  "answer": "The direct answer extracted from the report",
  "supporting_context": "Key passages from the report that support this answer"
}

## Guidelines
1. Focus on directly answering the research question
2. Be concise but comprehensive
3. Include relevant evidence and context
4. Maintain factual accuracy
\end{lstlisting}

\clearpage

\section{Additional Details}
\label{sec:additional}

\subsection{Instantiation 2: Open Deep Research}
\label{sec:instantiation2}

\textbf{The Outer Loop (Supervisor).}
The outer loop operates on the global research state. It operates over a sequence of discrete steps $t = 1,2, ..., T$. We define the components as follows:
\begin{itemize}[leftmargin=5mm,nosep,itemsep=2pt]
    \item \emph{Information Acquisition ($\pi_{\text{query}}^{\text{outer}}$):} The outer loop thinking strategy is parameterized by a backbone LLM. It generates an action $a_t \in \mathcal{S}^{\text{outer}} \cup \mathcal{T}^{\text{outer}}$, where $\mathcal{S}^{\text{outer}}$ is the space of research topics and $\mathcal{T}^{\text{outer}}$ is the termination signal.
    \item \emph{Search Engine ($\beta_{\text{engine}}^{\text{outer}}$):} If $a_t \in \mathcal{S}^{\text{outer}}$, the agent executes the inner loop to generate synthesized reports $i_t \sim \text{InnerLoop}(\cdot \mid a_t)$.
    \item \emph{Information Compression ($\pi_\text{sum}^\text{outer}$):} The outer loop does not conduct this step. Therefore, $h_t = \mathbb{1}(i_t)$, where $\mathbb{1}$ is the identity function.
    \item \emph{Inference ($\pi_{\text{update}}^{\text{outer}}$):} The outer belief is updated by appending the synthesized report returned by the inner-process: $\mathbf{b}_{t+1} = \text{Append}(\mathbf{b}_t, a_t, h_t)$.
\end{itemize}

\textbf{The Inner Loop (Researcher).}
The inner loop realizes the engine $\beta_{\text{engine}}^{\text{outer}}$ for the outer loop. It operates over a sequence of discrete steps $k = 1,2, ..., K$. We define the components as follows:
\begin{itemize}[leftmargin=5mm,nosep,itemsep=2pt]
    \item \emph{Information Acquisition($\pi_{\text{query}}^{\text{inner}}$):} A reactive LLM policy that generates an action $a_k \in \mathcal{S}^{\text{inner}} \cup \mathcal{T}^{\text{inner}}$, where $\mathcal{S}^{\text{inner}}$ represents search queries and $\mathcal{T}^{\text{inner}}$ is the local stop signal.
    \item \emph{Search Engine ($\beta_{\text{engine}}^{\text{inner}}$):} If $a_k \in \mathcal{S}^{\text{inner}}$, the search engine returns raw information $i_k \sim \beta_{\text{engine}}^{\text{inner}}(\cdot \mid a_k)$.
    \item \emph{Information Compression ($\pi_{\text{sum}}^{\text{inner}}$):} An information processing policy condenses raw information $i_k$ into a concise observation $h_k \sim \pi_{\text{sum}}^{\text{inner}}(\cdot \mid i_k)$.
    \item \emph{Inference ($\pi_{\text{update}}^{\text{inner}}$):} The inner belief is updated by appending the concise observation into the current belief state: $\mathbf{b}_{k+1} = \text{Append}(\mathbf{b}_k, a_k, h_k)$.
\end{itemize}

\subsection{Details about Evaluation Metrics}
\label{app: eval}

\subsubsection{Probabilistic Interpretation}

For the one-hot encoding vector $\mathbf{Y}$, let $\mathbf{y}_i$ denote the one-hot vector representation of the outcome $y^{(i)}$ from run $i$. Since $\mathbf{y}_i$ is a standard basis vector, it has unit norm ($\|\mathbf{y}_i\|_2 = 1$). Consequently, the squared Euclidean distance between two runs simplifies to a discrete metric:
\begin{equation}
    \|\mathbf{y}_i - \mathbf{y}_j\|^2 = 
    \begin{cases} 
        0 & \text{if } y^{(i)} \equiv y^{(j)} \\
        2 & \text{if } y^{(i)} \not\equiv y^{(j)}
    \end{cases}
\end{equation}
Substituting this into Eq.~\eqref{eq:tv_estimator}, the estimator reduces to the proportion of discordant pairs across all possible comparisons:
\begin{equation}
    \widehat{\text{TV}}(\mathbf{Y}) = \frac{1}{n(n-1)} \sum_{i \neq j} \mathbb{I}(y^{(i)} \not\equiv y^{(j)}).
\end{equation}
This value represents the empirical probability that two independently sampled runs will produce different answers.

For binary vectors $\mathbf{B},\mathbf{C}$, the squared distance between two normalized vectors relates directly to their Cosine similarity:
\begin{equation}
    \|\tilde{\mathbf{x}}_i - \tilde{\mathbf{x}}_j\|^2 = 2 - 2 (\tilde{\mathbf{x}}_i \cdot \tilde{\mathbf{x}}_j).
\end{equation}

Substituting the cosine identity $\|\tilde{\mathbf{x}}_i - \tilde{\mathbf{x}}_j\|^2 = 2 - 2 (\tilde{\mathbf{x}}_i \cdot \tilde{\mathbf{x}}_j)$ into Eq.~\eqref{eq:tv_estimator}, the estimator simplifies to the average complement of the cosine similarity:
\begin{equation}
    \widehat{\text{TV}} = 1 - \frac{1}{n(n-1)} \sum_{i \neq j} (\tilde{\mathbf{x}}_i \cdot \tilde{\mathbf{x}}_j).
\end{equation}

We use findings as an illustrative example. Let $S_i$ denote the set of finding indices discovered in run $i$. Consider two runs $i$ and $j$. Assuming the volume of findings is consistent across runs such that $|S_i| = |S_j| = k$, the dot product of the normalized vectors approximates the conditional probability of recurrence:
\begin{align}
    \tilde{\mathbf{x}}_i \cdot \tilde{\mathbf{x}}_j 
    &= \frac{|S_i \cap S_j|}{\sqrt{|S_i||S_j|}} \notag \\
    &= \frac{|S_i \cap S_j|}{k} \notag \\
    &= P(\text{finding} \in S_j \mid \text{finding} \in S_i).
\end{align}

Taking the expectation over all pairs, we arrive at the following probabilistic interpretation:
\begin{equation}
    \mathbb{E}[\widehat{\text{TV}}] = 1 - P(\text{finding} \in S_j \mid \text{finding} \in S_i).
\end{equation}

\subsubsection{Incorporating Semantic Geometry}
In our primary analysis, we treat answers and findings as categorical distributions — disagreeing on `cat' vs. `dog' incurs the same penalty as `cat' vs. `car'. If we consider the semantic geometry of answers (i.e., answer `cat' is closer to `dog' than to `car' in some embedding space) and
findings, our TV framework naturally extends to this setting by redefining the realization vectors using embedding distances.

For a semantic space equipped with a similarity metric $s(u, v) \in [0, 1]$, we can generalize our vector construction.
\begin{itemize}[leftmargin=5mm,nosep,itemsep=2pt]
    \item \emph{Answers:} Instead of a one-hot vector $\mathbf{e}_k$, the realization $\mathbf{y}_i$ becomes the dense embedding vector $\mathbf{v}_{y^{(i)}}$. The Euclidean distance $\|\mathbf{y}_i - \mathbf{y}_j\|^2$ then directly captures semantic divergence (e.g., small for `cat' vs. `dog', large for `cat' vs. `car').
    
    \item \emph{Findings/Citations:} If the global set of findings is $\{f_1, \dots, f_K\}$, the $k$-th entry of $\mathbf{b}_i$ can be defined as $\max_{f \in S_i} s(f_k, f)$.
    For example, suppose the global findings are $\{\text{cat}, \text{dog}, \text{car}\}$ with $s(\text{cat}, \text{dog})=0.8$.
    \begin{itemize}
        \item A run finding only ``cat" yields $\mathbf{b} = [1, 0.8, 0]$.
        \item A run finding ``dog" and ``car" yields $\mathbf{b} = [0.8, 1, 1]$.
    \end{itemize}
\end{itemize}
In this formulation, TV can be used to measure the variance of the agent's output in the semantic embedding space rather than the discrete symbolic space.

\subsubsection{Sensitivity to Relative Scale}
A potential concern with using $L_2$ normalization is that it discards the absolute magnitude of the vectors (i.e., the total number of findings). However, this property is desirable for measuring stability: it ensures that our metric focuses on the relative consistency of the information retrieved, rather than the raw volume.

To illustrate this, consider how the metric penalizes a ``single omission' error — where one run misses exactly one finding discovered by another run. Intuitively, missing 1 finding out of 10 is a more severe failure of consistency compared to missing 1 finding out of 100. Our normalized $\widehat{\text{TV}}$ captures this distinction. As an example, we compare two scenarios where Run 1 generates a set of findings $S_1$ and Run 2 generates $S_2$, with an intersection size of $|S_1 \cap S_2|$.

\begin{itemize}[leftmargin=5mm,nosep,itemsep=2pt]
    \item \emph{Case A:}
    The agent discovers 100 findings in Run 1, but misses one in Run 2 ($|S_1|=100, |S_2|=99, |S_1 \cap S_2|=99$).
    \[ 
    \widehat{\text{TV}}_A = 1 - \frac{99}{\sqrt{100 \times 99}} = 1 - 0.995 = \mathbf{0.005} 
    \]
    
    \item \emph{Case B:}
    The agent discovers 10 findings in Run 1, but misses one in Run 2 ($|S_1|=10, |S_2|=9, |S_1 \cap S_2|=9$).
    \[ 
    \widehat{\text{TV}}_B = 1 - \frac{9}{\sqrt{10 \times 9}} = 1 - 0.949 = \mathbf{0.051} 
    \]
\end{itemize}

Although the absolute disagreement is identical (1 finding difference) in both cases, the estimator assigns a penalty to Case B that is an order of magnitude larger ($\approx 10\times$). This confirms that $\widehat{\text{TV}}$ correctly scales.

\clearpage

\subsection{Mitigation in API Setting through Temperature Tuning}
\label{app:temp_tune}

\begin{table}[h]
\centering
\caption{Stochasticity and Accuracy Metrics across Sampling Temperatures ($\lambda$) under API Setting.}
\label{tab:temp_stochasticity_comparison}
\begin{tabular}{l|ccc}
\toprule
\textbf{Metric} & \textbf{$T=0$} & \textbf{$\lambda=0.5$} & \textbf{$\lambda=1.0$} \\
\midrule
$\widehat{\mathrm{TV}}(\mathbf{Y})$  & 0.700 & 0.670 & 0.620 \\
$\widehat{\mathrm{TV}}(\mathbf{B})$ & 0.828 & 0.836 & 0.834 \\
$\widehat{\mathrm{TV}}(\mathbf{C})$& 0.621 & 0.608 & 0.612 \\
\midrule
Accuracy            & 0.28  & 0.24  & 0.24  \\
\bottomrule
\end{tabular}
\end{table}

\subsection{Example for Structured Output}
\label{app:structured}

\begin{lstlisting}[caption=Structured Reasoning Output]
<reasoning>{\n  \"established_facts\": [],\n  \"open_questions\": [\n    \"Which countries were in the top 10 with the lowest GPI scores in both 2022 and 2023 according to Vision of Humanity?\",\n    \"Of those countries, which had a gun homicide rate below 0.20 per 100,000 population in both 2022 and 2023 according to World Population Review?\",\n    \"Are the GPI scores and gun homicide rate data available and comparable for the same set of countries across both years?\"\n  ],\n  \"next_search_directions\": [\n    \"Search for the Global Peace Index (GPI) 2022 and 2023 reports from Vision of Humanity to identify the top 10 most peaceful countries (lowest scores).\",\n    \"Cross-reference the list of countries that appear in both the 2022 and 2023 top 10 lowest GPI scores.\",\n    \"Search World Population Review's data for gun homicide rates in 2022 and 2023 for those countries.\",\n    \"Filter countries with gun homicide rates less than 0.20 per 100,000 in both years.\"\n  ],\n  \"contradictions_or_uncertainties\": [\n    \"Potential mismatch in country naming or classification between Vision of Humanity and World Population Review.\",\n    \"Possible discrepancies in reporting years or data collection methods between the two sources.\",\n    \"Uncertainty about whether World Population Review reports gun homicide rates for all countries listed in the GPI top 10.\"\n  ]\n}</reasoning>
\end{lstlisting}

\end{document}